\documentclass{article} 
\usepackage{iclr2019_conference,times}

\usepackage{vae-as}

\title{Improve Variational AutoEncoder with \\ Auxiliary SOFTMAX MultiClassifier}

\author{
Yao Li \\
Institute of Big Data \\
Fudan University \\
\texttt{18110980007@fudan.edu.cn} \\
}

\iclrfinalcopy 
\begin{document}

\maketitle

\begin{abstract}
As a general-purpose generative model architecture, 
VAE has been widely used in the field of image and natural language processing. 
VAE maps high dimensional sample data into continuous latent variables with unsupervised learning. 
Sampling in the latent variable space of the feature, VAE can construct new image or text data. 
As a general-purpose generation model, 
the vanilla VAE can not fit well with various data sets and networks with different structures.
Because of the need to balance the accuracy of reconstruction 
and the convenience of latent variable sampling in the training process,
VAE often has problems known as ``posterior collapse'',
and images reconstructed by VAE are also often blurred.
In this paper, we analyze the main cause of these problem,
which is the lack of control over mutual information between the sample variable 
and the latent feature variable during the training process.
To maintain mutual information in model training,
we propose to use the auxiliary softmax multi-classification network structure 
to improve the training effect of VAE, named VAE-AS. 
We use MNIST and Omniglot data sets to test the VAE-AS model. 
Based on the test results, 
It can be show that VAE-AS has obvious effects on the mutual information adjusting 
and solving the posterior collapse problem.
\end{abstract}

\section{INTRODUCTION}
In recent years, generative model has played an important role in the field of machine learning 
such as natural language processing (NLP) and image processing.
Compared with supervised learning for pattern recognition and classification, 
Unsupervised and semi-supervised generation model is used more for image synthesis, automatic text response, etc.
Variational Autoencoder (VAE)~\citep{kingma2013auto} is a typical generator model, 
which encodes the observed samples by an encoder, then restore the observation sample through a decoder.
VAE's feature space has good properties, compared with observed variables, 
the compressed latent variables encoding has the effect of dimensionality reduction and noise immunity.
Latent variables capture and summarize the feature of the information better.
At the same time, VAE can also generate new images or text 
by sampling in latent variable space and restoring by the decoder.
VAE provides a good, scalable computing framework for tasks such as migration learning and more.

In the variational processing of VAE,
in order to solve the problem that the posterior distribution is unintegrable,
an approximation method of variational inference is used to calculate.
VAE suppose a simple prior distribution of latent variables,
and use KL-Divergence to compare the difference between the posterior and the prior distribution of the latent variable.
For machine learning, the requirement for prior distribution can be seen as a regular term,
which plays a role in simplifying calculations and optimizing the coding space.
There are two contradictory aspects in VAE's learning goals.
On the one hand, 
we require that the code generated by the encoder can objectively and truly reflect the distribution of actual data.
On the other hand,
we also require the distribution of the code to be as simple as possible so that we can easily generate new samples.
~\citep{bowman2015generating, chen2016variational} mentions the problem of the posterior latent variable collapse in VAE.
The decoder tends to ignore the latent variable $\rvz$ and generate $\rvx$ directly.
As mentioned in~\citep{dosovitskiy2016generating}, 
VAE tends to produce blurred images when using complex natural image datasets.

Now, the methods for solving these problems mainly include three types.
The first method is to cancel hypothesis of the simple normal prior distribution in VAE 
and instead use a more complex and accurate distribution to describe the latent code distribution.
When we know more about the actual data distribution, we can use a more explicit prior distribution, 
such as ~\citep{xu2018spherical} using the Mises-Fisher distribution instead of the standard normal prior distribution.
The MCMC method is more accurate than variational inference but requires multiple iterations,
how to combine the advantages of both and apply it to machine learning is proposed as a topic
~\citep{salimans2015markov,hoffman2017learning}.
The second type of approach is to start with the training process of the model.
In the early stages of model training, 
the encoder could not establish the association between the latend variable $\rvz$ and the observed variable $\rvx$,
this led to the occurrence of posterior latent variable collapse.
In these papers~\citep{sonderby2016train, higgins2017beta, burgess2018understanding},
the weight coefficient is added before the KL-Divergence term of VAE loss function, 
and an annealing mechanism is used to make the KL-Divergence 
$D_{KL}[q_{\phi}(\rvz \mid \rvx)||p_{\theta}(\rvz)]$ play a small role at the beginning.
When the encoder establishes an association between $\rvz$ and $\rvx$, 
it gradually increases the alignment requirements between the approximate posterior and the prior distribution.
~\citep{he2019lagging} changes the way VAE trains encoder and decoder at the same time,
use sequence training to update the decoder and encoder parameters.
The third method is to use the compensation mechanism in the original optimization goal of VAE.
~\citep{zhao2017infovae, dieng2018avoiding,ma2019mae} added a regular term of the mutual information 
$\mathcal{I}_{q_{\phi}}(\rvz,\rvx)$ to the loss function of VAE.
This method is highly targeted,
but the mutual information $\mathcal{I}_{q_{\phi}}(\rvz,\rvx)$ is difficult to calculate 
and generally depends on Monte Carlo sampling for estimation.
~\citep{tolstikhin2017wasserstein, zhao2017infovae} 
used the NMD-based penalty to calculate the difference between the two distributions.
~\citep{dieng2018avoiding} proposes a method based on skip model to optimize mutual information.
~\citep{ma2019mae} estimates mutual information based on the mutual KL-Divergence divergence(MPD).

In this article, 
(1) Based on the calculation process of VAE, we analyzed the causes of encode collapse and blurry problem in detail.
(2) In the framework of vanilla VAE, 
we propose a novel method to estimate the mutual information $\mathcal{I}_{q_{\phi}}(\rvz,\rvx)$ 
and the marginal KL-Divergence $D_{KL}(q_{\phi}(\rvz)\ ||\ p_{\theta}(\rvz))$ 
by the auxiliary softmax MultiClassifier, noted as VAE-AS.
(3) we test the validity of VAE-AS using MNIST and Omniglot data sets,
and evaluate the quality of hidden variable space from various aspects.

\section{Background}
\subsection{VAE}

Modeling the joint distribution of the latent variable $\rvz$ and the observed variable $\rvx$,
VAE calculates the margin distribution of $\rvx$.
Consider the sample data set $\mX=\{\vx^{(i)}\}_{i=1}^{N}$ consisting of $N$ i.i.d. samples.
Introduce the latent variable $\rvz$, the joint distribution of $\rvx$ and $\rvz$ is $p_{\theta}(\rvx,\rvz)$.
The true distribution of $\rvx$ can be obtained by computing the margin distribution 
$p_{\theta}(\rvx)=\int p_{\theta}(\rvx \mid \rvz)p_{\theta}(\rvz) \text{d}\rvz$.
However, this integral is often intractable, 
so VAE uses the approximate distribution $q_{\phi}(\rvz \mid \rvx)$ to solve the problem by variational inference.

The margin likelihood function can be written as the sum of the likelihood function for each sample point,
$\text{log}p_{\theta}(\vx^{(1)},\cdots,\vx^{(N)})=\sum_{i=1}^{N}\text{log }p_{\theta}(\vx^{(i)})$.
Combin with  Jensen Inequality, for each sample $\vx^{(i)}$,
\begin{align}
    \text{log} p_{\theta}(\vx^{(i)}) &= \text{log} \int p_{\theta}(\vx^{(i)},\rvz) \text{d}\rvz 
    =\text{log} \mathbb{E}_{q_{\phi}(\rvz\mid \vx^{(i)})}\left[
    \frac{p_{\theta}(\vx^{(i)},\rvz)}{q_{\phi}(\rvz\mid \vx^{(i)})} \right] \notag \\
& \geq \mathbb{E}_{q_{\phi}(\rvz\mid \vx^{(i)})} \left[
\text{log} p_{\theta}(\vx^{(i)},\rvz)- \text{log}q_{\phi}(\rvz\mid \vx^{(i)}) \right] 
    = \mathcal{L}_{ELBO}(\theta,\phi;\vx^{(i)})
    \label{eq-2-1-1}
\end{align}
where, joint distribution is $p_{\theta}(\rvx,\rvz) = p_{\theta}(\rvx\mid \rvz)p_{\theta}(\rvz)$, 
the evidence lower bound can be written as
\begin{align}
    \mathcal{L}_{ELBO}(\theta,\phi;\mX) & = \sum_{i=1}^{N} \left 
    [\mathbb{E}_{q_{\phi}(\rvz\mid \vx^{(i)})} \left[ \text{log} p_{\theta}(\vx^{(i)}\mid \rvz) \right] 
    - D_{KL}\left[ q_{\phi}(\rvz\mid \vx^{(i)})\ ||\ p_{\theta}(\rvz) \right ] \right ]
    \label{eq-2-1-2}
\end{align}
It can be seen from the \eqref{eq-2-1-2} that the loss function of VAE consists of two parts.
One part of the loss is the observation $\vx^{(i)}$ reconstruction error, 
and the other part is the difference between the posterior distribution $q_{\phi}(\rvz|\rvx)$ 
and the prior distribution $p_{\theta}(\rvz)$.

The prior distribution of the latent variable $\rvz$ chosen by VAE is denoted as $p_{\theta}(\rvz)$,
General assumption is a multidimensional standard normal distribution 
$\mathcal{N}(\rvz; \boldsymbol{0},\mI)$ to facilitate sampling of the model.
The complete calculation of VAE includes two processes of encoding and decoding.
During the encoding process,
the encoder training fit a conditional posterior distribution of $\rvz$ after $\rvx$ is observed,
noted as $q_{\phi}(\rvz\mid \rvx)$. 
We use KL-Divergence $D_{KL}(q_{\phi}(\rvz\mid \rvx)\ ||\  p_{\theta}(\rvz))$ 
to evaluate the similarity of a prior and conditional posterior distributions of latent variable.
That is, the posterior distribution generated by the encoder is required 
to be as close as possible to a standard normal distribution.
The decoding process is to sample $\vz \sim p_{\theta}(\rvz)$ from the prior distribution,
then use the decoder $p_{\theta}(\rvx|\vz)$ to generate a $\vx$ refactoring $\hat{\vx}$.
Compare the difference between $\vx$ and $\hat{\vx}$ as a loss function, 
calculating the gradient to update parameters $\theta, \phi$.

The calculation of $\text{log }p_{\theta}(\vx^{(i)},\rvz)$ in VAE depends on sampling.
Specifically, for each sample $\vx^{(i)}$, VAE encoder fits the variables $\mu(\vx^{(i)})$ and $\sigma(\vx^{(i)})$,
it is used to define the posterior distribution $q_{\phi}(\rvz\mid \vx^{(i)})=
\mathcal{N}(\rvz;\mu(\vx^{(i)}),\sigma(\vx^{(i)}))$ 
of the latent variable $\rvz$.
Then VAE completes $L$ samples and calculates
$$\tilde{\mathcal{L}}(\theta,\phi;\mX) \simeq \frac{1}{N}\sum_{i=1}^{N} \left [
    \frac{1}{L}\sum_{l=1}^{L} \left( \text{log} p_{\theta}(\vx^{(i)}\mid \vz^{(l)}) \right)
- D_{KL}\left[ q_{\phi}(\rvz \mid \vx^{(i)})\ ||\ p_{\theta}(\rvz) \right ] \right ]$$ 

In order to ensure the continuity of the loss calculation,
enable the inverse gradient descent training process.
VAE uses the technique of reparameterization during the sampling process.
First sample the random variable $\varepsilon^{(l)}$ from $\mathcal{N}(0,1)$,
then let $\vz^{(l)}=\mu(\vx^{(i)})+\sigma(\vx^{(i)}) \odot \varepsilon^{(l)}$.

\subsection{Disadvantages of VAE}
For VAE, the latent variable space needs to have sufficient feature coding ability to express complex real distributions.
The sample space is often discrete, 
but the encoder needs to smooth and fill the missing parts between discrete samples to generate new pictures or text.
At the same time, the coding space must be conducive to sampling and likelihood calculations, 
which requires that the prior distribution $p_{\theta}(\rvz)$ is simple.

The VAE encoder encodes each discrete sample $\vx^{(i)}$ of $\mX$ into a continuous latent random variable $\rvz$.
It can be seen that $\vx^{(i)}$ is mapped to a continuous area of distribution defined by 
$q_{\phi}(\rvz \mid \vx^{(i)})= \mathcal{N}(z;\mu(\vx^{(i)}), \sigma(\vx^{(i)}))$.
Intuitively, in the calculation of the equation \eqref{eq-2-1-2},
VAE hopes that the difference between the distribution of $\rvz$ generated by each $\vx^{(i)}$ encoding is large 
to reduce the error of reconstruction.
But minimizing the KL-Divergence term pushes $q_{\phi}(\rvz|\vx^{(i)})$ to a uniform standard normal distribution.
These two optimization directions are contradictory,
When VAE tries to reduce the reconstruction error,
it is bound to increase the difference of the code distribution $q_{\phi}(\rvz\mid \vx^{(i)})$,
then KL-Divergence will increase.
When the reduced reconstruction error is smaller than the increased KL-Divergence, 
VAE directly reduces the KL-Divergence without the error reduction.

Design an experiment,
when feeding a completely random set of images to VAE for learning,
we will find that the KL-Divergence will soon tend to $0$.
Because the image does not contain any rules, 
the VAE reconstruction error is optimized so small that it does not start at all.


There are two cases of coding collapse.
One is that in the early stage of model training,
because the parameters optimization of the decoder have not been completed 
or because the randomness of the training samples is too strong,
strong KL-Divergence constraints will result in a high optimization constraint threshold,
models are easier to optimize only KL-Divergence.
In a more general view, the text~\citep{he2019lagging} rewrites the ELBO to
$$\mathcal{L}_{ELBO}(\theta,\phi;\mX) = \sum_{i=1}^{N} \left [ 
\text{log } p_{\theta}(\vx^{(i)}) - 
D_{KL}\left[ q_{\phi}(\rvz\mid \vx^{(i)})\ ||\ p_{\theta}(\rvz\mid \vx^{(i)}) \right ]\right ] $$
The first term in R.H.S can be seen as the maximum log likelihood of the decoder's distribution.
The KL-Divergence of the second term characterizes the degree of similarity between 
$p_{\theta}(\rvz \mid \rvx)$ and $q_{\phi}(\rvz \mid \rvx)$.
When the VAE is trained in the beginning, 
the encoder is not perfect,
the correlation between $\rvz$ and $\rvx$ is not strong and can be regarded as an independent variable.
When $p_{\theta}(\rvz)$ is simple, 
minimizing the KL-Divergence requires $q_{\phi}(\rvz | \rvx)$ to be directly locked to $p_{\theta}(\rvz)$,
which is, $\rvz$ is independent with $\rvx$, $q_{\phi}(\rvz_d \mid \vx^{(i)})=q_{\phi}(\rvz_d) \sim \mathcal{N}(\rvz_d;0,1)$,
similarly $p_{\theta}(\rvz_d \mid \vx^{(i)})= p_{\theta}(\rvz_d) \sim \mathcal{N}(\rvz_d;0 , 1)$,
$D_{KL}$ item disappears on this element.

Another type of collapse is due to the lack of correlation in the model encoding, 
or the need of correlation for reconstruction is small,
it also causes the fact that the feature space dimension exceeds the active number of latent variable.
Image and language coding often have autocorrelation, 
such as the value of a pixel in an image and the value of surrounding pixels.
In the case of a strong decoder, using the autocorrelation information, 
$\rvx$ can be recovered without $\rvz$, 
As can be seen from the Table~\ref{tab1}, 
with the increase of layers of Decoder network, the decoding ability of decoder is continuously enhanced.
Encoder $q_{\phi}(\rvz \mid \rvx)$  constructs $\mu(\vx)$ according to $\vx$,
The number of dimensions with larger variances of $\mu(\vx)$ continues to decline.
This shows that the higher the complexity of the Decoder, 
then the less information is used in the information contained in $\rvz$.

\begin{table}
  \centering
  \caption{The number of active units of $\rvz$, changes with number of decoder network layers}
  \label{tab1}
  \begin{tabular}{l|c|c|c|c|c}
    \toprule[1pt]
    \hhline
    Decoder Layers      & 1 & 2 & 3 & 4 & 5 \\ \hline
    $\text{Cov}_{p(\rvx)} \left( \mathbb{E}_{q_{\phi}(\rvz \mid \rvx)}[z_{d}] \right) \geq 0.01$ & 23 &21 & 17 &12 &8 \\ 
    \bottomrule[1pt]
  \end{tabular}
\end{table}

From the above, in order to prevent the posterior collapse, 
we need to increase the correlation between $\rvz$ and $\rvx$ in the joint distribution 
to prevent independent of $\rvz$ and $\rvx$.
The correlation in the joint distribution is mainly controlled by mutual information.
The problem of blurred images is also related to mutual information. 
Mentioned in the article~\citep{zhao2017towards},
the optimal solution to reconstruction loss of decoder for a given $q_{\phi}$ 
is $\mathbb{E}_{q_{\phi}(\rvx \mid \vz)}\left[ \rvx \right]$.
That is to say, the reconstructed image of VAE is the expectation of a set of images,
which is related to the smoothness of the latent variable distribution.
The mutual information is directly related to the entropy of $q_{\phi}(\rvx|\vz)$.
We prefer to think the blurry as a trade-off rather than a problem.
In the article~\citep{shwartz2017opening, alemi2016deep} on information bottlenecks, the neuron network will have a decrease in mutual information $I_{q_{\phi}}(Z,X)$ in the later stage of optimization training.
It is actually an optimization that reduces overfitting.
when the mutual information is reduced, the smoothness of the model is higher,
and the anti-noise and fitting between samples is better.
As the mutual information increases, the error of model reconstruction will be smaller,
but the generalization performance of the model will be worse.
Therefore, for mutual information, what we need more is to better control according to the model task,
rather than simply increasing or decreasing.


\section{Method}
\subsection{Mutual Information}
According to~\citep{hoffman2016elbo, dieng2018avoiding},
the KL-Divergence term in the ELBO of the VAE can be expressed as (See the appendix for details).
$$\frac{1}{N} \sum_{i=1}^{N} D_{KL}(q_{\phi}(\rvz \mid \rvx) ||p_{\theta}(\rvz)) 
= \mathcal{I}_{q_{\phi}}(\rvz,\rvx)+D_{KL}(q_{\phi}(\rvz) || p_{\theta}(\rvz))$$
where $\mathcal{I}_{q_{\phi}}$ is the mutual information between $\rvx$ and $\rvz$, 
measure the correlation between $\rvx$ and $\rvz$. 
And $D_{KL}(q_{\phi}(\rvz) || p_{\theta}(\rvz))$ is the KL-Divergence between code marginal distribution $q_{\phi}(\rvz)$
and prior distribution $p_{\theta}(\rvz)$.

In order to control the mutual information and the marginal KL-Divergence more precisely, 
rewrite the ELBO:
\begin{align}
    \hat{\mathcal{L}}(\theta,\phi;\mX) &\simeq \frac{1}{N}\sum_{i=1}^{N} \left [
    \frac{1}{L}\sum_{l=1}^{L} \left( \text{log} p_{\theta}(\vx^{(i)}\mid \vz^{(l)}) \right) \right ]
    - \alpha \mathcal{I}_{q_{\phi}}(\rvz, \rvx) 
    - \beta D_{KL}\left[ q_{\phi}(\rvz)\ ||\ p_{\theta}(\rvz) \right ] 
    \label{eq-3-1-1}
\end{align}
Where $\alpha$ and $\beta$ are Lagrange multiplier, 
By adjusting $\alpha$, $\beta$,
we can adjust the performance of VAE-AS to meet different optimization requirements.

The calculation of mutual information and marginal KL-Divergence involves solving the margin distributions of $q_{\phi}(\rvz)$.
For research objects such as images and texts, the dimensions are often large and the probability of the samples are both small.
Select sample each time, 
the distribution of the sample $\{\vx^{(1)},\cdots,\vx^{(N)}\}$ can be considered as a discrete Categorical distribution
$Cat(c^{(1)},\cdots,c^{(N)})$, where $c^{(1)}=\cdots=c^{(N)}=\frac{1}{N}$.
We note this category distribution as $q_{\phi}(\rvx)$.
Since $q_{\phi}(\rvx=\vx^{(i)})=\frac{1}{N}$, the empirical distribution can be derived as
$q_{\phi}(\rvz) = \sum_{i=1}^{N} q_{\phi}(\rvz \mid \vx^{(i)}) q_{\phi}(\rvx=\vx^{(i)})
=\frac{1}{N} \sum_{i=1}^{N} q_{\phi}(\rvz \mid \vx^{(i)})$.
$q_{\phi}(\rvz)$ is called aggregated posterior distribution~\citep{makhzani2015adversarial},
which is difficult to obtain analytical solutions.
Sampling to calculate the estimate of mutual information and marginal KL-Divergence
is generally~\citep{hoffman2016elbo, dieng2018avoiding}.

We explore mutual information $\mathcal{I}_{q_{\phi}}$ from the perspective of empirical distribution.
According to the definition of mutual information,
the mutual information $\mathcal{I}_{q_{\phi}}$ can be rewritten as:
\begin{align}
    \mathcal{I}_{q_{\phi}}(\rvz,\rvx)= \mathbb{E}_{q_{\phi}(\rvx)}\left[ \int q_{\phi}(\rvz \mid \rvx) 
    \text{log } \frac{q_{\phi}(\rvz \mid \rvx)}{q_{\phi}(\rvz)} \text{d}\rvz \right] 
    \label{eq-3-1-2}
\end{align}

For the integral term in the form \eqref{eq-3-1-2},
since $q_{\phi}(\rvz) = \frac{1}{N} \sum_{i=1}^{N} q_{ \phi}(\rvz \mid \vx^{(i)})$.
we can select $M$ sample to approximate in each training batch,
$q_{\phi}(\rvz) \simeq \frac{1}{M} \sum_{i=1}^{M} q_{\phi}(\rvz \mid \vx^{(i)})$.
But unfortunately, this method of estimation is biased,
we will use another method to calculate mutual information.

\subsection{MINE}
\begin{theorem}[Donsker-Varadhan representation]
\label{theo-1}
 The KL divergence between any two distributions $\mathbb{P}$ and 
$\mathbb{Q}$, 
with $\mathbb{P} \ll \mathbb{Q}$, admits the following dual representation~\citep{donsker1983asymptotic}
\begin{equation}
    D_{KL}(\mathbb{P} || \mathbb{Q}) = \sup_{T:\Omega \rightarrow \mathbb{R}} \mathbb{E}_{\mathbb{P}}[T] - 
    \text{log }(\mathbb{E}_{\mathbb{Q}}[\text{exp }(T)])
\end{equation}
where the supremum is taken over all functions $T$ such that the two expectations are finite.
\end{theorem}


From Theorem~\ref{theo-1}, MINE~\citep{belghazi2018mutual} has been proposed 
as a method for estimating mutual information using a neural network.
If $p(\vx),q(\vx)$ are the density functions of $\mathbb{P}, \mathbb{Q}$ respectively,
given any subclass $\mathcal{F}$ of such functions, then have
\begin{equation}
        \mathcal{I}(\rvx, \rvz) \geq \sup_{T \in \mathcal{F}} \big \{\mathbb{E}_{p(\rvx,\rvz)}[T(\rvx, \rvz)] 
        - \text{log }\big(\mathbb{E}_{p(\rvx)p(\rvz)}\big[\text{exp }\big(T(\rvx, \rvz)\big)\big]\big)\big\}
\end{equation}
where $T \in \mathcal{F}$ can be fitted with a neuron network $T_{\psi}$ parameterized $\psi \in \Psi$.
The expectations in the above lower-bound can then be estimated by Monte-carlo sampling.
After sampling, we can get
\begin{equation}
        \hat{\psi} = \arg\sup_{\psi \in \Psi} \big \{ \mathbb{E}_{q(\rvx,\rvz)}[T_{\psi}(\rvx, \rvz)] 
        - \text{log }\big(\mathbb{E}_{q(\rvx)q(\rvz)}\big[\text{exp }\big(T_{\psi}(\rvx, \rvz)\big)\big]\big)\big \}
\end{equation}
then we obtain the Mutual Information Neural Estimator (MINE):
\begin{equation}
        \hat{\mathcal{I}}(\rvx, \rvz) = \mathbb{E}_{q(\rvx,\rvz)}[T_{\hat{\psi}}(\rvx, \rvz)] 
        - \text{log }\big(\mathbb{E}_{q(\rvx)q(\rvz)}\big[
        \text{exp }\big(T_{\hat{\psi}}(\rvx, \rvz)\big)\big]\big)
    \label{eq-MINE}
\end{equation}

\subsection{Auxiliary SOFTMAX multiclassifier}
In this paper,
we add an auxiliary softmax multi-classifier based on VAE to solve the mutual information $\mathcal{I}_{q_{\phi}}$,
which is called VAE-AS.
For mutual information, we can write
\begin{align}
    \mathcal{I}_{q_{\phi}}(\rvz,\rvx) = \text{log }N - \mathbb{E}_{q_{\phi}(\rvz)}\left[ 
    \mathcal{H}\left( q_{\phi}(\rvx \mid \rvz) \right) \right]
    \label{eq-3-2-0}
\end{align}
In order to calculate the mutual information,
we need to solve $q_{\phi}(\rvx \mid \rvz)$ and the entropy of $q_{\phi}(\rvx \mid \rvz)$.
According to Bayes' theorem,
$q_{\phi}(\rvx\mid \rvz) = \frac{q_{\phi}(\rvz\mid \rvx)q_{\phi}(\rvx)}{q_{\phi}(\rvz)}$, 
and $q_{\phi}(\rvx=\vx^{(i)}) =\frac{1}{N}$, $q_{\phi}(\rvz)=\frac{1}{N}\sum_{i=1}^{N} q_{\phi}(\rvz\mid \vx^{(i)})$, 
we know,
\begin{align}
    q_{\phi}(\rvx=\vx^{(i)} \mid \rvz) = \frac{q_{\phi}(\rvz \mid \vx^{(i)})}{\sum_{j=1}^{N}q_{\phi}(\rvz \mid \vx^{(j)})}
    \label{eq-3-2-1}
\end{align}

As can be seen from the discussion in the previous section,
it is still difficult to directly calculate \eqref{eq-3-2-1}.
We use the auxiliary softmax multi-classifier to fit $q_{\phi}(\rvx\mid \rvz)$.
First number each sample, let $\{\vx^{(1)},\ldots,\vx^{(N)}\} \mapsto \{1, \cdots,N\}$. 
Equivalent to determining a label $\ve^{(i)}$  for each sample $\vx^{(i)}$,
Where $\ve^{(i)}$ is the one-hot variable $[0,\dots,0,1,0,\dots,0]$ with a 1 at position $i$.
i.e.
$$\rve_{j}^{(i)}=\left \{\begin{aligned}
        &1\quad if\ j=i \\
        &0\quad if\ j\neq i
    \end{aligned}\right .$$

\begin{figure}[h]
\begin{center}
\includegraphics[width=0.7\linewidth]{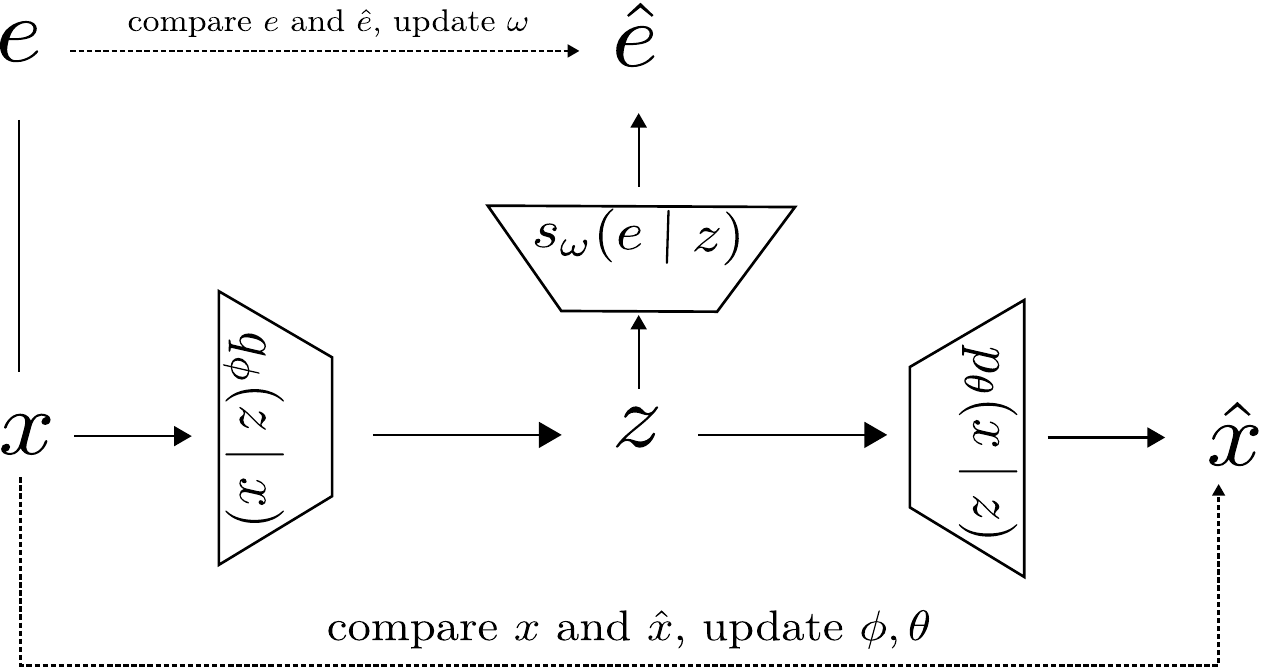}
\end{center}
\caption{VAE-AS Network Structure }
\label{fig1}
\end{figure}

Based on the structure of VAE,
we add a multi layer neural network $s_{\omega}(\hat{\rve} \mid \rvz)$ to fit the distribution $q_{\phi}(\rvx \mid \rvz)$ .
The network $s_{\omega}(\hat{\rve} \mid \rvz)$ is trained in a supervised manner,
we use the one-hot variable $\ve$ as label,
the activation function of the network 
$s_{\omega}(\hat{\rve} \mid \rvz)$ is the softmax function.

For each sampled $\vz^{(l)}$, after a multi-layered neuron network transformation,
let $\vh^{(l)}_{0} = \vz^{(l)}$, $\vh^{(l)}_{t} = g(\omega_{t-1}^T \vh^{(l)}_{t-1})$,
where $g(\cdot)$ is activation function, such as $tanh$ or $ReLU$.
The output $\vh^{(l)}_o = \omega_t^T \vh^{(l)}_{t}$,
where $\vh^{(l)}_{o}$ is $N$ dimensions vector, The element $i$ of $\vh^{(l)}_{o}$ is $\vh^{(l)}_{o,i}$.
The auxiliary softmax multi-classifier can be written as:
\begin{align}
    s_{\omega}(\hat{\rve}_i=1 \mid \vz^{(l)}) = \frac{\text{exp }(\vh^{(l)}_{o,i})}{\sum_{j=1}^{N} \text{exp }(\vh^{(l)}_{o,j})}
\end{align}
Since the serial number can be seen as a random tag $\ve^{(i)}$ allocated for each sample $\vx^{(i)}$,
The loss function is the cross entropy of the softmax function, i.e.
\begin{align}
    loss_s = -\frac{1}{L}\sum_{l=1}^{L} \sum_{j=1}^{N} \ve_{j}^{(i)} \text{log } s_{\omega}(\hat{\rve}_j =1 \mid \vz^{(l)})
    \label{eq-3-2-3}
\end{align}
Minimizing the cross entropy loss function, 
We hope $s_{\omega}(\hat{\rve} \mid \rvz)$ will gradually converge to $q_{\phi}(\rvx \mid \rvz)$.
After completing the fitting of the distribution $q_{\phi}(\rvx \mid \rvz)$, 
we can calculate the entropy $\mathcal{H}\left( q_{\phi}(\rvx \mid \rvz) \right)$ .
\begin{align}
\mathcal{H}\left( q_{\phi}(\rvx \mid \rvz) \right) 
= -\sum_{i=1}^{N} q_{\phi}(\rvx=\vx^{(i)}\mid \rvz) \text{log }q_{\phi}(\rvx=\vx^{(i)} \mid \rvz)
    \label{eq-3-2-2}
\end{align}
then we can calculate the mutual information $\mathcal{I}_{q_{\phi}}(\rvz,\rvx)$ with \eqref{eq-3-2-0}.

For Calculating $D_{KL}\left( q_{\phi}(\rvz)\ ||\ p_{\theta}(\rvz) \right)$,
we need calculate $q_{\phi}(\rvz)$, and then estimate with Monte Carlo Sampling.
After using the auxiliary softmax to estimate
$q_{\phi}(\rvx = \vx^{(i)} \mid \rvz)$, 
we can estimate $q_{\phi}(\rvz)$ by Bayesian formula.
\begin{align}
    q_{\phi}(\rvz) = \frac{q_{\phi}(\rvz \mid \vx^{(i)})}{Nq_{\phi}(\rvx = \vx^{(i)} \mid \rvz)}
    \label{eq-3-2-4}
\end{align}
After estimating the mutual information 
and the marginal divergence 
we can calculate ELBO as \eqref{eq-3-1-1}.

\subsection{VAE-AS and MINE}
\begin{theorem}
\label{theo-2}
VAE-AS and MINE are equivalent in the optimization goal and estimation of mutual information.
\end{theorem}

\begin{proof}
    Step 1: 
    prove that VAE's optimization goal is equivalent to MINE's optimization goal.
    The output variable $\vh$ of classifier $s_{\psi}$ is $N$ dimensions vector.
    its $i$ component corresponds to the observation vector $\vx^{(i)}$,
    $h^{(l)}$ can be thought as a function of $\vx^{(i)}$ and $\vz^{(l)}$, 
    $\vh^{(l)}_i = T_{\psi}(\vx^{(i)}, \vz^{(l)})$,
    then from \eqref{eq-3-2-3} the optimization goal can be written:
    \begin{equation*}
        \max \frac{1}{NL} \sum_{i=1}^{N} \sum_{l=1}^{L} \text{log }
        \frac{\text{exp }(T_{\psi}(\vx^{(i)}, \vz^{(l)}))}{\sum_{j=1}^{N} \text{exp }(T_{\psi}(\vx^{(j)}, \vz^{(l)}))}
    \end{equation*}
    Decomposing the $\text{log}$ function and 
    adding the constant $\text{log }\frac{1}{N}$ does not affect optimization,
    the optimization goal is equivalent to
    \begin{equation*}
            \max \frac{1}{NL}\sum_{i=1}^{N} \sum_{l=1}^{L} T_{\psi}(\vx^{(i)}, \vz^{(l)}) 
            - \frac{1}{NL}\sum_{i=1}^{N} \sum_{l=1}^{L} \text{log } 
            \Big(\frac{1}{N}\sum_{j=1}^{N}\text{exp }\big(T_{\psi}(\vx^{(j)}, \vz^{(l)})\big)\Big)
        \end{equation*}
    Since sample $\vz^{(l)}$ from $q(\rvz \mid \vx^{(i)})$,
    the first term will converge to $\mathbb{E}_{q(\rvx, \rvz)}(T_{\psi}(\rvx, \rvz))$.
    For the second term, about $z^{(l)}$, 
    we need to calculate $T_{\psi}$ for each $x^{(j)}$, 
    so $z^{(l)}$ is independent of $x^{(j)}$, the second term will converge to
    $\text{log }(\mathbb{E}_{q(\rvx)q(\rvz)}\left[\text{exp }(T_{\psi}(\rvx, \rvz))\right])$.

    Step 2:
    When $\psi = \hat{\psi}$, $s_{\hat{\psi}}(\rvz)$ can be seen as fit of $q(\rvx \mid \rvz)$.
    From \eqref{eq-3-2-0}, 
    \begin{equation*}
            \hat{\mathcal{I}}_{q}(\rvx,\rvz) = \mathbb{E}_{q(\rvz)}\Big[ 
            \sum_{i=1}^{N}\hat{q}(\vx^{(i)}\mid \rvz)T_{\hat{\psi}}(\vx^{(i)},\rvz) \\
            - \sum_{i=1}^{N} \hat{q}(\vx^{(i)}\mid \rvz)
                \text{log }\Big(\frac{1}{N}\sum_{j=1}^{N}\text{exp }
            \big(T_{\hat{\psi}}(\vx^{(j),\rvz})\big)\Big)\Big]
    \end{equation*}
    where 
    the first item will converge to
    $\mathbb{E}_{q(\rvx, \rvz)}T_{\hat{\psi}}(\rvx, \rvz)$,
    the rest is converged to $\text{log }(\mathbb{E}_{q(\rvx) q(\rvz)}
    \left [\text{exp }(T_{\hat{\psi}}(\rvx, \rvz)) \right] )$.

\end{proof}

\subsection{VAE-AS optimization}
Calculating mutual information by the formula \eqref{eq-3-2-2} requires traversal of all samples,
which is computationally burdensome.
The following proposition will give an approximate estimate.

\begin{theorem}
\label{prop-1}
For each $\vx^{(i)}$, 
Note the probability of multi-classifier prediction error
$\mathbb{P}(\hat{\rve}_i \neq \rve_i \mid \rvz)=s_{\omega}(\hat{\ve}^{(i)} \neq \ve^{(i)} \mid \rvz)$
as $\mathbb{P}_e^{(i)}$, let $\mathbb{P}_e = \frac{1}{N}\sum_{i=1}^{N} \mathbb{P}_{e}^{(i)}$.
For VAE-AS, let
$$\hat{\mathcal{I}}(\rvz,\rvx)=\text{log }N + \mathbb{E}_{q_{\phi}(\rvz)}\left[ 
    \mathbb{P}_e^{(i)}
    \text{log } \mathbb{P}_e^{(i)}
    + (1-\mathbb{P}_e^{(i)})
    \text{log } (\mathbb{P}_e^{(i)})
\right ]$$
be an estimate of $\mathcal{I}_{q_{\phi}}(\rvz,\rvx)$.
There is 
$$\hat{\mathcal{I}}(\rvz,\rvx) - \mathcal{I}_{q_{\phi}}(\rvz,\rvx) \leq \mathbb{P}_e \text{log }N$$.
\end{theorem}

\begin{proof}

    According to the network structure, $\rvx \rightarrow \rvz \rightarrow \hat{\rve}$ constitutes a Markov chain.
    From Fano's Inequality we know, for each $\vx^{(i)}$,
    $$\mathcal{H}(\mathbb{P}_e^{(i)})+\mathbb{P}_e^{(i)} \text{log }N 
    \geq \mathcal{H}(q_{\phi}(\rvx=\vx^{(i)} \mid \rvz))$$
    $$\text{log }N - \mathcal{H}(\mathbb{P}_e^{(i)}) - \mathbb{P}_e^{(i)} \text{log }N \leq 
    \text{log }N - \mathcal{H}(q_{\phi}(\rvx=\vx^{(i)} \mid \rvz))$$
    where, 
    $$\mathcal{H}(\mathbb{P}_e^{(i)})=
    - s_{\omega}(\hat{\rve}^{(i)} = \rve^{(i)} \mid \rvz) \text{log } s_{\omega}(\hat{\rve}^{(i)}=\rve^{(i)} \mid \rvz)
    - s_{\omega}(\hat{\rve}^{(i)} \neq \rve^{(i)} \mid \rvz) \text{log } s_{\omega}(\hat{\rve}^{(i)} \neq \rve^{(i)} \mid \rvz)
    $$.
    From \eqref{eq-3-2-0} we can get,
    $$\hat{\mathcal{I}}(\rvz,\rvx) - \mathcal{I}_{q_{\phi}}(\rvz,\rvx) \leq \mathbb{P}_e \text{log }N$$

\end{proof}

Intuitively, Theorem {prop-1} simplifies the entropy calculation in a Multinomial distribution
by using the part of the error probability as a whole.
Therefore, a bias of $\mathbb{P}_e \text{log }N$ will be generated.
With optimization, the predicted error rate $\mathbb{P}_e$ will gradually decrease,
The gap between $\hat{\mathcal{I}}(\rvz,\rvx)$ and $\mathcal{I}_{q_{\phi}}(\rvz,\rvx)$ will be gradually reduced to $0$.

When the sample size is large, the cost of softmax multi-classification calculation is large.
$q_{\phi}(\rvz)$ needs to calculate the sum of all categories probability,
i.e. the normalized part of the denominator of the softmax function.
To reduce the computational load of the softmax function,
similar to the Huffman tree method used in~\citep{morin2005hierarchical},
we uses bianry tree.

For each $\vx^{(i)}$ we define tags $\ve^{(i)}$,
$\ve$ is one-hot random vector, 
w.r.t. sampling $\vz^{(l)}$, the output vector of classifier network $s_{\omega}$ is $\vh^{(l)}$. 
Construct a $V$ layers binary tree,
each leaf node, representing a component of the one-hot vector,
every two child node has one parent node, finally shrink to the root node.
Each layer before the leaf node is noted as $\vt^{(1)}, \dots,\vt^{(V)}$,
let $\vt^{(l,v)}=\omega^{(v)}\vh^{(l)}$, where $\omega^{(v)}$ is parameter matrix.
The vector length is increased layer by layer, with up to $2^v$ nodes per layer.
It can be seen that for $N$ samples, $V >= \text{log}_{2}\frac{N}{2}$.
Except for leaf nodes, each node component $t^{(v)}_{j}$ 
has one left and one right child nodes (edge nodes maybe has one child).
The path from the parent node to the child node is denoted as $d_{j}^{(v)}$, 
and when $d_{j}^{(v)}=1$, the node $t_{j}^{(v)}$ select the path to left child,
select right child when $d_{j}^{(v)}=0$.

using the sigmoid activation function $\sigma(\cdot)$ to calculate the probability of two paths.
$$\mathbb{P}(d^{(v)}_{j}=1 \mid \vz^{(l)})=\sigma(t^{(l,v)}_j)$$
correspondingly we know $\mathbb{P}(d_{j}^{(v)}=0 \mid \vz^{(l)})=1-\sigma(t^{(l,v)}_j)$.
For an one-hot vector $\ve^{(i)}$,
there is a unique path on the binary tree pointing to its corresponding leaf node.
The node through which the path passes is $t_{j_{1}}^{(l,1)},\dots,t_{j_{V}}^{(l,V)}$,
The path of the selected child node for each node $t_{j_{v}}^{(l,v)}$ is $d_{j_{v}}^{(l,v)}$.
Then the fit probability of the classifier is
$$s_{\omega}(\hat{\ve}=\ve^{(i)} \mid \vz^{(l)})=\prod_{v=1}^{V}\mathbb{P}(d_{j_v}^{(l,v)} \mid \vz^{(l)})$$
where $d_{j_v}^{(l,v)}$ is $0$ or $1$.
Compare $\hat{\ve}$ and $\ve$ to construct the loss function,
the parameters $\omega$ in the binary tree are optimized to fit the softmax function.


Using the softmax function approximation, the computational complexity is $O(hN)$, 
where $h$ is the classifier output feature number.
Based on multidimensional binary labels, the computational complexity is optimized to $O(h\text{ log}_2N)$.

In the structure of VAE, 
when the distributions mapped by one sample $\vx^{(i)}$ 
and another sample $\vx^{(j)}$ are far away in the latent variable space,
i.e. $\vx^{(i)}$ and $\vx^{(j)}$ are quite different, 
Even if the identity of the two is same, 
it will not affect the system's fit of $q_{\phi}(\rvx \mid \rvz)$.
In other words,
the number of label categories we randomly select for the sample can be smaller than the actual number of samples.
If we select the number of tag categories as $V$,
refer to the entropy $\mathcal{H}(q_{\phi}(\rvx \mid \rvz))$ for \eqref{eq-3-2-2},
The maximum entropy of the fitted distribution is $\text{log }V$,
which is the case where the fitting distribution is completely indistinguishable from the uniform probability value of $\rvx$.
From $0 < \mathcal{I}_{q_{\phi}(\rvz,\rvx)} < \text{log }N$ (see Appendix for details),
The maximum value of the mutual information $\mathcal{I}_{q_{\phi}}$ 
and the maximum value of $\mathbb{E}_{q_{\phi}(\rvz)}\mathcal{H}(q_{\phi}(\rvx | \rvz))$ is same, i.e. $\text{log }N$.
The maximum mutual information 
that we can provide by using the $V$ tag category through the auxiliary softmax multi-classifier is $\text{log }V$.
Since the $\text{log}$ function grows very slowly,
selecting $V$ categories smaller than $N$ will not affect the VAE-AS's mutual information so much.

\section{Related Work}
The method for reducing the collapse of latent variables given in~\citep{dieng2018avoiding}
is also developed around adding the model's mutual information $\mathcal{I}_{q_{\phi}}(\rvx,\rvz)$.
The method is divided into two steps,
First, the Skip Network is used to 
make the latent variables participate in the calculation of each hidden layer output of the subsequent decoder.
Then the author also adds constraints of mutual information in the loss function.
The author discusses that the mutual information of Skip Network is greater than the mutual information of VAE, 
i.e. $\mathcal{I}_{q_{\phi}}^{SKIP-VAE}(\rvz,\rvx) \geq \mathcal{I}_{q_{\phi}}^{VAE}(\rvz,\rvx)$

The text~\citep{ma2019mae} optimizes 
the learning of the feature layer by adding regularization regular terms to the ELBO of the variation.
The author proposed mutual KL-Divergence between a pair of data to measure the diversity of posteriors.
$$MPD = \mathbb{E}_{\vx^{(1)},\vx^{(2)} \sim p(\rvx)}\left[ 
D_{KL}\left( q_{\phi}(\rvz \mid \vx^{(1)}) ||q_{\phi}(\rvz \mid \vx^{(2)}) \right) \right]$$
In addition to optimizing the ELBO of the VAE, the optimization goal also requires maximizing $MPD$.
In the article~\citep{ma2019mae}, the author gives the proof.
The mutual posterior diversity (MPD) is actually a symmetric version of the mutual information 
$\mathcal{I}_{q_{\phi}}(\rvz,\rvx)$.
Therefore, the method of~\citep{ma2019mae} is the same as~\citep{dieng2018avoiding},
both of them add regular items of mutual information based on VAE to improve the optimization effect.

~\citep{tolstikhin2017wasserstein, zhao2017infovae} mentions the use of Maximum Mean Discrepancy (MMD) 
to estimate the distance between distributions $q_{\phi}(\rvz)$ and $p_{\theta}(\rvz)$,
the focus remains on the estimate of $q_{\phi}(\rvz)$.
MMD uses a kernel function to estimate the distance between different sample points,
which is similar in form to the softmax's solution.

All of the above methods require the calculation of $q_{\phi}(\rvz)$ in the training batch.
If the batch\_size of the training batch is too small, 
the estimated bias of the mutual information $\mathcal{I}_{q_{\phi}}$ will be larger.
When the number of batches is large,
such as estimating the distribution difference by MMD,
if all $N$ samples are used as batch calculations,
The calculation of the $O(N^2)$ kernel function is required,
and the cost of the calculation is relatively large.

In this paper, VAE-AS removes the batch constraints by constructing a softmax classification network,
and uses the parameters of the classification network to save the differences between different sample codes.
Through the calculation of Hierarchical Softmax, we can effectively reduce the number of complex exponential operations.

\section{Empirical Results}
\subsection{Dataset}
To test the actual effect of VAE-AS, 
we selected two reference data sets of image, 
MNIST~\citep{lecun1998gradient} and Omniglot~\citep{lake2013one,burda2015importance}.
The MNIST data set includes 70,000 binary black and white handwritten digital images,
for MNIST, 55,000 for training and 10,000 for testing.
The Omniglot data set contains $1623$ different handwritten characters from 50 different alphabets. 
The training set has $24345$ samples, and test set has $8069$ samples.
Each characters was drawn online  by $20$ different people.
The image size of both data sets is $28 \times 28$

\subsection{Estimate Mutual Information and Marginal Divergence}
Firstly, we use of VAE-AS to estimate mutual information $\mathcal{I}_{q_{\phi}}(\rvz, \rvx)$
and marginal divergence $D_{KL}(q_{\phi}(\rvz)\ ||\ p_{\theta}(\rvz))$.
For the dataset MNIST~\citep{lecun1998gradient},
Construct encoders, decoders and classifiers using MLP,
the number of layers is $2$, $3$, $2$,
the number of hidden layers is chosen to be $500$,
the dimension of $\rvz$ is set to $40$, the mini-batches is set to $100$
and the iteration epochs is set to $120$.

From the model, for each sample $\vx^{(i)}$, 
$\mu(\vx^{(i)})$ and $\sigma(\vx^{(i)})$ is generated,
Using Monte Carlo method to estimate $\mathcal{I}_{q_{\phi}}(\rvz, \rvx)$ 
and $D_{KL}(q_{\phi}(\rvz)\ ||\ p_{\theta}(\rvz))$,
the method is same as ~\citep{hoffman2016elbo}.
Set sample number of Monte Carlo to $S$, 
sampling $55000, 10000, 5000, 1000, 500, 100$ to estimate mutual information and marginal divergence.

Set the number of softmax labels to the full sample size of training set of MNIST $55000$,
using VAE-AS to estimate mutual information $\mathcal{I}_{q_{\phi}}(\rvz, \rvx)$
and marginal divergence $D_{KL}(q_{\phi}(\rvz)\ ||\ p_{\theta}(\rvz))$.
The test is notted as $\text{VAE-AS}_{55000}$.
To compare the results, we set the number of softmax labels to $10000$,
performe a comparision test $\text{VAE-AS}_{10000}$.
The test results are shown in Figure~\ref{estimate-mi-md}:

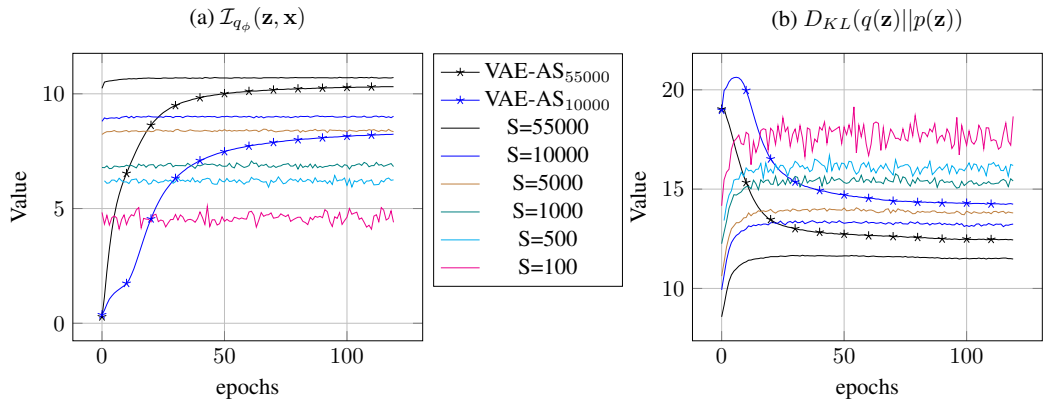
\begin{figure}[htbp]
\resizebox{1.0\textwidth}{!}{
\begin{tikzpicture}[baseline]
\begin{axis}[
title={(a) $\mathcal{I}_{q_{\phi}}(\rvz, \rvx)$},
xlabel={epochs},
ylabel={Value},
grid=major,
legend entries={{$\text{VAE-AS}_{55000}$}, {$\text{VAE-AS}_{10000}$}, S=55000, S=10000, S=5000, S=1000, S=500, S=100},
legend pos=outer north east,
]
\addplot [black, mark=star, mark repeat=10] table[col sep=comma, x=epochs, y=MI] {data/31vae-sm55000.test};
\addplot [blue, mark=star, mark repeat=10] table[col sep=comma, x=epochs, y=MI] {data/31vae-sm10000.test};
\addplot [black] table[col sep=comma, x=epochs, y=MI] {data/31vae-sm-eval-55000.test};
\addplot [blue] table[col sep=comma, x=epochs, y=MI] {data/31vae-sm-eval-10000.test};
\addplot [brown] table[col sep=comma, x=epochs, y=MI] {data/31vae-sm-eval-5000.test};
\addplot [teal] table[col sep=comma, x=epochs, y=MI] {data/31vae-sm-eval-1000.test};
\addplot [cyan] table[col sep=comma, x=epochs, y=MI] {data/31vae-sm-eval-500.test};
\addplot [magenta] table[col sep=comma, x=epochs, y=MI] {data/31vae-sm-eval-100.test};
\end{axis}
\end{tikzpicture}%
\begin{tikzpicture}[baseline]
\begin{axis}[
title={(b) $D_{KL}(q(\rvz) || p(\rvz))$},
xlabel={epochs},
ylabel={Value},
grid=major,
]
\addplot [black, mark=star, mark repeat=10] table[col sep=comma, x=epochs, y=MD] {data/31vae-sm55000.test};
\addplot [blue, mark=star, mark repeat=10] table[col sep=comma, x=epochs, y=MD] {data/31vae-sm10000.test};
\addplot [black] table[col sep=comma, x=epochs, y=MD] {data/31vae-sm-eval-55000.test};
\addplot [blue] table[col sep=comma, x=epochs, y=MD] {data/31vae-sm-eval-10000.test};
\addplot [brown] table[col sep=comma, x=epochs, y=MD] {data/31vae-sm-eval-5000.test};
\addplot [teal] table[col sep=comma, x=epochs, y=MD] {data/31vae-sm-eval-1000.test};
\addplot [cyan] table[col sep=comma, x=epochs, y=MD] {data/31vae-sm-eval-500.test};
\addplot [magenta] table[col sep=comma, x=epochs, y=MD] {data/31vae-sm-eval-100.test};
\end{axis}
\end{tikzpicture}
}
\caption{Estimate mutual information and marginal KL-Divergence.}
\label{estimate-mi-md}
\end{figure}

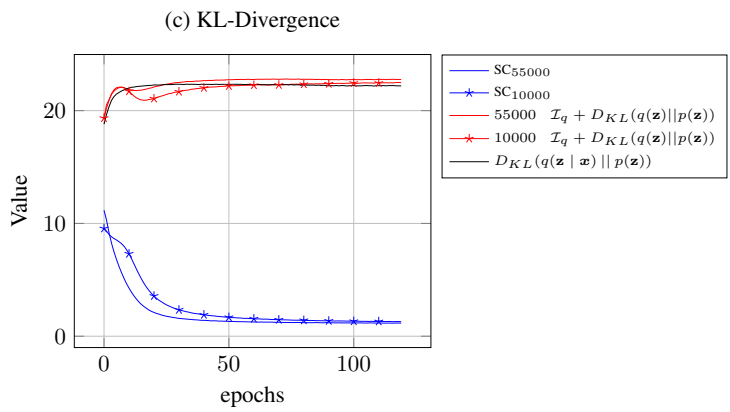
\begin{figure}[htbp]
\resizebox{0.7\textwidth}{!}{
\begin{tikzpicture}
\begin{axis}[
title=(c) KL-Divergence,
xlabel={epochs},
ylabel={Value},
grid=major,
legend entries=
{\tiny{SC${}_{55000}$},
\tiny{$\text{SC}_{10000}$}, 
\tiny{$55000\quad \mathcal{I}_{q}+D_{KL}(q(\rvz) || p(\rvz))$}, 
\tiny{$10000\quad \mathcal{I}_{q}+D_{KL}(q(\rvz) || p(\rvz))$},
\tiny{$D_{KL}(q(\rvz \mid \vx)\ ||\ p(\rvz))$}},
legend cell align=left,
legend pos=outer north east,
]
\addplot [blue] table[col sep=comma, x=epochs, y=SC] {data/31vae-sm55000.test};
\addplot [blue, mark=star, mark repeat=10] table[col sep=comma, x=epochs, y=SC] {data/31vae-sm10000.test};
\addplot [red] table[col sep=comma, x=epochs, y expr=\thisrow{MI}+\thisrow{MD}] {data/31vae-sm55000.test};
\addplot [red, mark=star, mark repeat=10] table[col sep=comma, x=epochs, y expr=\thisrow{MI}+\thisrow{MD}] {data/31vae-sm10000.test};
\addplot [black] table[col sep=comma, x=epochs, y=KL] {data/31vae-sm55000.test};
\end{axis}
\end{tikzpicture}
}
\caption{Estimate KL-Divergence.}
\label{estimate-kl}
\end{figure}

From the Figure~\ref{estimate-mi-md} we can see,
for the Monte Carlo method, 
as the number of samples drops,
whether the mutual information $\mathcal{I}_{q_{\phi}}(\rvz, \rvx)$
or marginal divergence $D_{KL}(q_{\phi}(\rvz)\ ||\ p_{\theta}(\rvz))$, 
will produce large bias.
And in the case of a small sample, 
there are a lot of fluctuations in the curve, 
which shows that Monte Carlo method will introduce a large variance due to sampling.
Generally limited by computing resources, we can't choose too much for mini-batch size during training.
Using VAE-AS to fit $q_{\phi}(\rvx = \vx^{(i)}\mid \rvz)$,
then estimate mutual information $\mathcal{I}_{q_{\phi}}(\rvz, \rvx)$
and marginal divergence $D_{KL}(q_{\phi}(\rvz)\ ||\ p_{\theta}(\rvz))$.
The estimated results are close to the large sample size Monte Carlo method.
Since VAE-AS only needs a large vocabulary,
the dimension of the classification label is small,
the demand for computation space can generally be satisfied.

The third figure depicts the sum of mutual information $\mathcal{I}_{q_{\phi}}(\rvz, \rvx)$
and marginal divergence $D_{KL}(q_{\phi}(\rvz)\ ||\ p_{\theta}(\rvz))$, 
compare with KL-Divergence $D_{KL}(q_{\phi}(\rvz \mid \rvx)\ ||\ p_{\theta}(\rvz))$.
It can be seen that the sum of mutual information and marginal divergence 
has a good convergence to the mean of KL-Divergence,
where SC is softmax cross-entropy.

\subsection{Results}
We tested the results of different parameters $\alpha$ and $\beta$ on VAE-AS optimization.
The encoder inference network $q_{\phi}(\rvz \mid \rvx)$ , the decoder generation network $p_{\theta}(\rvx \mid \rvz)$ 
and the classifier network $s_{\omega}(\rve \mid \rvx)$
are both constructed using multiple hidden layer neural networks MLP, 
the number of layers is $2$, $5$, $2$.
The number of hidden layers is chosen to be $500$.
the dimension of $\rvz$ is set to $40$, the mini-batches is set to $128$
and iteration epochs is set to $300$.
No other regular constraints are set.

In the test metrics, NLL${}_{test}$ is a non-negative likelihood based on the test set evaluation.
Using the method described in ~\citep{burda2015importance}, the number of samples is $4096$.
KL${}_{test}$ is the KL-Divergence $D_{KL}\left[ q_{\phi}(\rvz\mid \vx^{(i)})\ ||\ p_{\theta}(\rvz) \right ]$
based on the test set.
AU is the number of active units of the latent variable $\rvz$,
for this indicator we use the definition in ~\citep{burda2015importance}:
$$AU = \sum_{d=1}^{D} \mathbb{I}\{\text{Cov}_{p(\rvx)} 
\left( \mathbb{E}_{q_{\phi}(\rvz \mid \rvx)}[\rvz_d] \right) \geq \varepsilon \}$$
Where $\mathbb{I}(\cdot)$ is the indicator function, 
$D$ is the dimension of the latent variable $\rvz$, 
$\rvz_d$ is the element $d$ of $\rvz$, 
$\varepsilon$ is the judgment threshold and is set to $0.01$.
NNL${}_{train}$, KL${}_{train}$ is the reconstruction error and KL-Divergence based on the training set.
MI is mutual information $\mathcal{I}_{q_{\phi}}(\rvz, \rvx)$, 
MD is marginal divergence $D_{KL}(q_{\phi}(\rvz) || p_{\theta}(\rvz))$, estimated by VAE-AS during training.
SC is softmax cross-entropy.

\begin{table}[htb]
  \centering
  \caption{Adjust VAE-AS optimization effect by $\alpha$ and $\beta$}
  \label{tab2}
  \resizebox{1.0\linewidth}{!}{
  \begin{tabular}{l|c|c|c|c|c|c|c|c|c|c}
    \toprule[1pt]
    Method & $\alpha$ & $\beta$ & NNL${}_{test}$ & KL${}_{test}$ & AU & NNL${}_{train}$ & KL${}_{train}$ & MI & MD & SC\\ \hline
    VAE & - & - & 100.99 & 18.39 & 8 & 73.45 & 18.57 & - & - & -\\ 
    VAE-AS & 1.0 & 1.0 & 96.32 & 21.84 & 11 & 71.66 & 21.22 & 10.78 & 10.66 & 0.35 \\ \hline
    VAE-AS & 1.0 & 0.1 & 116.61 & 64.70 & 36 & 58.39 & 65.70 & 10.86 & 54.89 & 0.11 \\
    VAE-AS & 1.0 & 0.2 & 106.92 & 46.99 & 25 & 61.16 & 45.14 & 10.87 & 34.35 & 0.12 \\
    VAE-AS & 1.0 & 0.5 & 96.63 & 29.52 & 18 & 65.64 & 29.40 & 10.83 & 18.71 & 0.22 \\
    VAE-AS & 1.0 & 2.0 & 101.34 & 15.68 & 7 & 81.61 & 15.87 & 10.58 & 5.78 & 0.82 \\ 
    VAE-AS & 1.0 & 5.0 & 141.38 & 5.35 & 2 & 145.04 & 5.56 & 3.34 & 2.27 & 7.62 \\ \hline
    VAE-AS & 0.1 & 1.0 & 97.22 & 21.53 & 11 & 72.13 & 21.22 & 10.80 & 10.63 & 0.31 \\
    VAE-AS & 0.2 & 1.0 & 97.51 & 21.23 & 10 & 71.84 & 20.92 & 10.81 & 10.32 & 0.31 \\
    VAE-AS & 0.5 & 1.0 & 98.23 & 21.57 & 10 & 72.23 & 20.91 & 10.80 & 10.32 & 0.33 \\
    VAE-AS & 2.0 & 1.0 & 96.56 & 21.13 & 11 & 71.85 & 20.99 & 10.75 & 10.48 & 0.41 \\
    VAE-AS & 5.0 & 1.0 & 97.33 & 20.38 & 10 & 75.94 & 19.79 & 10.44 & 9.80 & 0.93 \\ 
    VAE-AS & 10.0  & 1.0 & 104.69 & 21.84 & 10 & 93.05 & 20.28 & 9.00 & 12.60 & 3.23 \\
    \bottomrule[1pt]
  \end{tabular}}
\end{table}

The test results are shown in the Table~\ref{tab2}.
The test is divided into two phases.
We fixed $\alpha=1$ and tested the effect of $\beta$ adjustment on model training and testing.
In the second phase, fix $\beta=1$, test the effect of the change of $\alpha$.

In VAE, the marginal divergence $D_{KL}(q_{\phi}(\rvz)\ ||\ p_{\theta}(\rvz))$
controls the overall shape of the latent variable distribution.
The larger $D_{KL}(q_{\phi}(\rvz)\ ||\ p_{\theta}(\rvz))$, 
the more scattered the latent variable distribution of $\rvz$,
i.e. the larger the variance of the latent variable distribution, 
more conducive to the training of the model.
From Table~\ref{tab2} we can see, if the mutual information $\mathcal{I}_{q_{\phi}}(\rvz, \rvx)$ is close, 
Increasing $D_{KL}(q_{\phi}(\rvz)\ ||\ p_{\theta}(\rvz))$is more advantageous for optimizing the reconstruction error.
But if the difference between the encoder marginal distribution and the prior distribution is far away,
Sampling and Calculating based prior distribution is poor quality.

When $\beta$ is fixed to $1$,
it can be seen from the Table~\ref{tab2} that the change in $\alpha$ affects the performance of the model generalization.
As can be seen from the \eqref{eq-3-2-0},
the size of the mutual information $\mathcal{I}_{q_{\phi}}(\rvz, \rvx)$ depends on
entropy $\mathcal{H}\left( q_{\phi}(\rvx \mid \rvz) \right)$, 
and $\mathcal{H}\left( q_{\phi}(\rvx \mid \rvz) \right) $ depends 
on degree of overlap between two encoding conditional distribution 
$q_{\phi}(\rvz \mid \vx^{(i)})$ and $q_{\phi}(\rvz \mid \vx^{(j)})$.
The overlap is more directly related to the variance of the encoder output $\sigma(\vx)$.
The smaller the $\sigma(\vx)$, the smaller the degree of overlap, 
the greater the mutual information.
So in the Table~\ref{tab2}, if $D_{KL}(q_{\phi}(\rvz)\ ||\ p_{\theta}(\rvz))$ is close. 
The larger the mutual information $\mathcal{I}_{q_{\phi}}(\rvz, \rvx)$,
the smaller the reconstruction error of the training.
But the decrease in training error caused by the increase of $\mathcal{I}_{q_{\phi}}(\rvz, \rvx)$ 
is only for one single sample.
Then it can't be reproduced on the test set, which is actually a case of overfitting.
If $\alpha$ is small, 
The model is more inclined to remember the appearance of each training sample than to observe the common law in the sample.
This is due to small overlap of the latent variable distribution $q_{\phi}(\rvz\mid \vx)$ caused by mutual information is large,
the decoder can not abstract and summarize the rules.
According to the information bottleneck theory~\citep{alemi2016deep},
we should minimize the mutual information to optimize model generalization while achieving the optimization goal.
However, the larger $\alpha$ value,
the mutual information $\mathcal{I}_{q_{\phi}} (\rvz, \rvx)$ included in the latent variable space is smaller,
and the number of active encodings is smaller,
the problem of posterior collapse is more likely to occur.
This is because the strong mutual information constraint limits the optimization threshold of the encoder.

\begin{figure}[htbp]
\centering
\subfigure[$\alpha=0.2$]{
\begin{minipage}[t]{0.3\linewidth}
\centering
\includegraphics[width=1.0\linewidth]{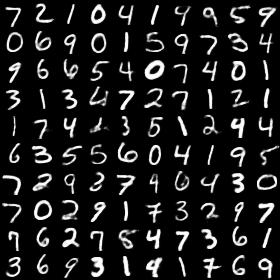}
\end{minipage}%
}%
\subfigure[Test Set]{
\begin{minipage}[t]{0.3\linewidth}
\centering
\includegraphics[width=1.0\linewidth]{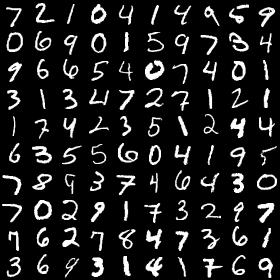}
\end{minipage}%
}%
\subfigure[$\alpha=5.0$]{
\begin{minipage}[t]{0.3\linewidth}
\centering
\includegraphics[width=1.0\linewidth]{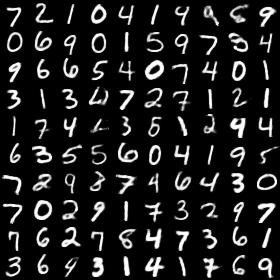}
\end{minipage}
}%

\subfigure[$\alpha=0.2$]{
\begin{minipage}[t]{0.3\linewidth}
\centering
\includegraphics[width=1.0\linewidth]{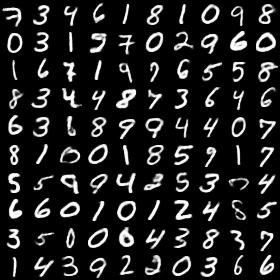}
\end{minipage}%
}%
\subfigure[Train Set]{
\begin{minipage}[t]{0.3\linewidth}
\centering
\includegraphics[width=1.0\linewidth]{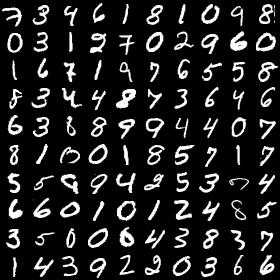}
\end{minipage}%
}%
\subfigure[$\alpha=5.0$]{
\begin{minipage}[t]{0.3\linewidth}
\centering
\includegraphics[width=1.0\linewidth]{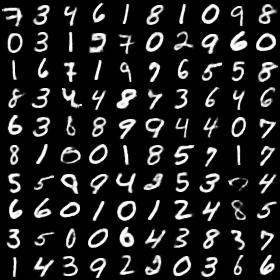}
\end{minipage}
}%
\centering
\caption{Compare the effects of image reconstruction when $\alpha=0.2$ and $\alpha=5.0$. 
(a)$\alpha=0.2$, reconstruction image in Test sets (b)original image in Test sets
(c)$\alpha=5.0$, reconstruction image in Test sets (d) $\alpha=0.2$, reconstruction image in train sets 
(e) original image in train sets (f) $\alpha=5.0$, reconstruction image in train sets}
\label{fig-5-2}
\end{figure}

The effect of $\alpha$ on the model can also be seen in the Figure~\ref{fig-5-2}.
The smaller $\alpha$ has less restriction on mutual information,
for training set samples, the model tends to reconstruct more accurate images.
Compared with the original image (e), 
the training reconstruction errors of $\alpha=0.2$ and $\alpha=5.0$ are both small, 
but the graph (d) shows more subtle differences in handwritten fonts than in graph (f). 
But obviously it will cause more test error, 
so the clarity of Figure (a) is higher than that of Figure (c),
but there are some obvious errors.
This is because when the mutual information is large, 
the $q_{\phi}(\rvz \mid \rvx)$ generated by different $\vx$ sample codes overlaps very little, 
and the model tends to restore the closest sample in the training set.
The reconstruction details that can be expressed in the training set are not reproducible in the test set, 
which is obviously an over-fitting situation.
Thus, a smaller $\alpha$ value will result in a sharper picture, 
avoiding the problem of the VAE model blurry.
However, it will result in poor generalization performance of the model.
This can be seen as a trade-off. 
If the continuity of ground-truth distribution is poor, 
then the model should choose a smaller $\alpha$ to get more mutual information in the latent variable space.
If the continuity of ground-truth distribution is better,
or if we need to interpolate in the discrete sampling training set,
then we need mutual information $\mathcal{I}_{q_{\phi}}(\rvz, \rvx)$ to be given more restrictions.

In the training process of VAE-AS, the values of $\mathcal{I}_{q_{\phi}}(\rvz, \rvx)$
and $D_{KL}(q_{\phi}(\rvz)\ ||\ p_{\theta}(\rvz))$ have mutual influence.
As can be seen from the Table~\ref{tab2}, if $\beta >1$,
The constraint requires that the latent variable distribution $q_{\phi}(\rvz)$ 
is closer to $p_{\theta}(\rvz)$, which is generally a standard normal distribution.
Often the overall shape of the data distribution will be more concentrated.
Then the overlap of latent variable distribution $q_{\phi}(\rvz \mid \rvx)$ is increased
due to the extrusion of the $\rvz$ latent variable space, mutual information is thus reduced.
With the decline of mutual information, the accuracy of decoder reconstruction will also decrease.
And with decrease of $\beta$, 
the contraint of marginal divergence $D_{KL}(q_{\phi}(\rvz)\ ||\ p_{\theta}(\rvz))$ is gradually relaxed.
The distribution of $q_{\phi}(\rvz)$
is more likely to be optimized according to the decoder's requirement to reduce the reconstruction error.
There will also be an increase in mutual information in the training.

In the Table~\ref{tab2}, when $\alpha$ is in the range of values from $0.1$ to $2.0$,
due to the constraint of mutual information upper bound, the change of mutual information is small.
When the $\alpha$ is small, although the mutual information
$\mathcal{I}_{q_{\phi}}(\rvz, \rvx)$ is already close to the upper limit bound $\text{log }N$,
But because the constraints are looser, encoder easily reduces the output variance $\sigma(\vx)$.
VAE-AS can maintain mutual information while reducing marginal divergence.
However, VAE-AS needs to be optimized for training in a smaller encoding space, 
so the training will be more difficult and the number of iterations of training will be more.


From the test results, 
we can see that the marginal KL-Divergence controls the overall distribution shape of the code distribution,
which determines the upper limit precision that the model can achieve.
Mutual information determines the degree of overlap between different latent variable, 
or the degree of smoothness of distribution between latent variable,
to a certain extent it controls the degree of overfitting.


\section{Conclusion}
This paper analyzes the reasons why VAE has a posterior collapse problem during training in detail.
A method of randomly assigning labels to samples and identifying labels by auxiliary softmax multi-classifiers 
is proposed to control mutual information of models.

Compared with image processing, posterior collapse is more likely to appear in the field of natural language process.
In NLP processing, decoding networks such as RNN are more widely used.
Therefore we will further test the effect of VAE-AS on NLP and RNN.
VAE-AS is able to accurately estimate mutual information and marginal divergence in training batches,
this facilitates many other machine learning tasks.
The information bottleneck theory discusses in detail the optimization process of mutual information. 
VAE is often used for disentanglement, which requires fine control of mutual information.
Subsequent we will use VAE-AS to conduct further research in these areas.



\bibliography{vae-as-cn.bib}
\bibliographystyle{iclr2019_conference}

\section{Appendix}
\subsection{ELBO Rewrite}
For the sample ${\vx^{(1)}, \cdots,\vx^{(N)}}$, consider $q_{\phi}(\rvx=\vx^{(i)})=\frac{1}{N}$.
\begin{align*}
    &\quad\  \mathcal{I}_{q_{\phi}}(\rvz,\rvx) \\ 
    &= \sum_{i=1}^{N} \int q_{\phi}(\rvz \mid \vx^{(i)})q_{\phi}(\vx^{(i)})\text{log}\left[ 
    \frac{q_{\phi}(\rvz \mid \vx^{(i)})q_{\phi}(\vx^{(i)})}{q_{\phi}(\vx^{(i)})q_{\phi}(\rvz)}\right] \text{d}\rvz \\
    &= \frac{1}{N}\sum_{i=1}^{N} \int q_{\phi}(\rvz \mid \vx^{(i)}) 
    \text{log}\frac{q_{\phi}(\rvz \mid \vx^{(i)})}{q_{\phi}(\rvz)} \text{d}\rvz \\ 
    \\
    &\quad\  \mathcal{I}_{q_{\phi}}(\rvx,\rvz) + D_{KL}\left( q_{\phi}(\rvz) || p_{\theta}(\rvz) \right) \\
    &= \mathcal{I}_{q_{\phi}}(\rvx,\rvz) + 
    \int q_{\phi}(\rvz) \text{log} \frac{q_{\phi}(\rvz)}{p_{\theta}(\rvz)} \text{d}\rvz \\
    &= \mathcal{I}_{q_{\phi}}(\rvx,\rvz) + \frac{1}{N} \sum_{i=1}^{N} \int q_{\phi}(\rvz \mid \vx^{(i)}) 
     \text{log} \frac{q_{\phi}(\rvz)}{p_{\theta}(\rvz)} \text{d}\rvz \\
     &= \frac{1}{N}\sum_{i=1}^{N}\int q_{\phi}(\rvz \mid \vx^{(i)}) 
     \text{log}\frac{q_{\phi}(\rvz \mid \vx^{(i)})}{p_{\theta}(\rvz)} \text{d}\rvz \\
     &= \frac{1}{N} \sum_{i=1}^{N} D_{KL}\left( q_{\phi}(\rvz \mid \vx^{(i)}) || p_{\theta} (\rvz) \right)
\end{align*}
\subsection{Mutual information value range}
\begin{align*}
    \mathcal{I}_{q_{\phi}}(\rvz,\rvx) &= \mathbb{E}_{q_{\phi}(\rvx,\rvz)}\left[ \text{log } 
    \frac{q_{\phi}(\rvx,\rvz)}{q_{\phi}(\rvx)q_{\phi}(\rvz)} \right] \\
    &= \mathbb{E}_{q_{\phi}(\rvz)}\left[ \int q_{\phi}(\rvx \mid \rvz) 
    \text{log } \frac{q_{\phi}(\rvx \mid \rvz)}{q_{\phi}(\rvx)} \text{d}\rvx \right] \\
&= \text{log }N - \mathbb{E}_{q_{\phi}(\rvz)} \left [\mathcal{H}\left( q_{\phi}(\rvx \mid \rvz) \right) \right ]
\end{align*}
When $\rvx$ is independent with $\rvz$,
$\mathbb{E}_{q_{\phi}(\rvz)}\left[ \mathcal{H}\left( q_{\phi}(\rvx \mid \rvz) \right) \right]= 
\mathcal{H}\left( q_{\phi}(\rvx) \right) = \text{log }N$, 所以$\mathcal{I}_{q_{\phi}}(\rvz,\rvx) =0$.
When $\rvx$ is mapped one-to-one with $\rvz$, 
$q_{\phi}(\rvx\mid \rvz)=1$, $\mathcal{H}_{q_{\phi}}(\rvx \mid \rvz)=0$, 
$\mathcal{I}_{q_{\phi}}(\rvz,\rvx) =\text{log }N$.
Thus,
$$0 \leq \mathcal{I}_{q_{\phi}}(\rvz,\rvx) = \text{log }N 
- \mathbb{E}\left[ \mathcal{H}\left( q_{\phi}(\rvx \mid \rvz) \right) \right]
\leq \text{log }N$$

\subsection{Results for Omniglot}
\begin{table}[hp]
  \centering
  \caption{Results of Omniglot Dataset}
  \label{tab3}
  \resizebox{1.0\linewidth}{!}{
  \begin{tabular}{l|c|c|c|c|c|c|c|c|c|c}
    \toprule[1pt]
    Method & $\alpha$ & $\beta$ & NNL${}_{test}$ & KL${}_{test}$ & AU & NNL${}_{train}$ & KL${}_{train}$ & MI & MD & SC\\ \hline
    VAE & - & - & 147.21 & 9.33 & 3 & 133.70  & 9.25 & - & - & -\\ \hline
    VAE-AS & 1.0 & 0.1 & 158.54 & 81.14 & 40 & 61.31 & 82.70 & 10.10 & 72.61 & 0.00 \\
    VAE-AS & 1.0 & 0.2 & 144.05 & 61.90 & 34 & 63.96 & 61.59 & 10.10 & 51.50 & 0.01 \\
    VAE-AS & 1.0 & 0.5 & 133.44 & 37.22 & 21 & 70.27 & 37.59 & 10.09 & 27.52 & 0.02 \\
    VAE-AS & 1.0 & 1.0 & 140.08 & 22.81 & 10 & 83.91 & 22.56 & 10.07 & 12.53 & 0.07 \\
    VAE-AS & 1.0 & 2.0 & 146.47 & 8.38 & 3 & 135.03 & 8.56 & 6.17 & 2.56 & 4.10 \\ 
    VAE-AS & 1.0 & 5.0 & 161.17 & 2.35 & 1 & 161.56 & 2.66 & 1.83 & 0.84 & 8.28 \\ 
    VAE-AS & 1.0 & 10.0 & 174.49 & 1.46 & 0 & 177.22 & 2.09 & 0.59 & 1.55 & 9.56 \\ \hline 
    VAE-AS & 0.1 & 1.0 & 136.82 & 23.23 & 11 & 82.03 & 23.45 & 10.07 & 13.41 & 0.06 \\
    VAE-AS & 0.2 & 1.0 & 140.38 & 21.17 & 9 & 86.22 & 21.49 & 10.07 & 11.45 & 0.06 \\
    VAE-AS & 0.5 & 1.0 & 137.52 & 21.57 & 10 & 84.68 & 22.13 & 10.07 & 12.10 & 0.07 \\
    VAE-AS & 2.0 & 1.0 & 140.60 & 21.16 & 9 & 86.02 & 21.27 & 10.06 & 11.25 & 0.08 \\
    VAE-AS & 5.0 & 1.0 & 133.10 & 20.97 & 10 & 86.42 & 21.31 & 10.04 & 11.35 & 0.13 \\ 
    VAE-AS & 10.0  & 1.0 & 175.63 & 41.15 & 7 & 145.96 & 39.76 & 4.22 & 36.84 & 7.18 \\
    \bottomrule[1pt]
  \end{tabular}}
\end{table}

\begin{figure}[htbp]
\centering
\subfigure[$\alpha=0.2$]{
\begin{minipage}[t]{0.3\linewidth}
\centering
\includegraphics[width=1.0\linewidth]{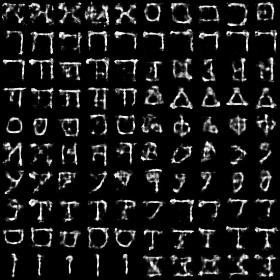}
\end{minipage}%
}%
\subfigure[Test Set]{
\begin{minipage}[t]{0.3\linewidth}
\centering
\includegraphics[width=1.0\linewidth]{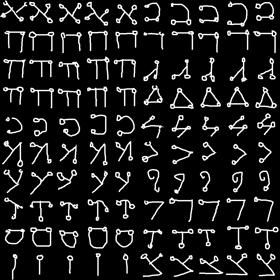}
\end{minipage}%
}%
\subfigure[$\alpha=5.0$]{
\begin{minipage}[t]{0.3\linewidth}
\centering
\includegraphics[width=1.0\linewidth]{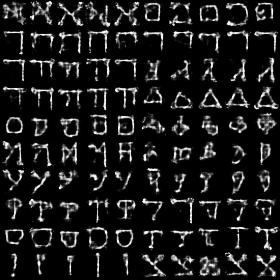}
\end{minipage}
}%

\subfigure[$\alpha=0.2$]{
\begin{minipage}[t]{0.3\linewidth}
\centering
\includegraphics[width=1.0\linewidth]{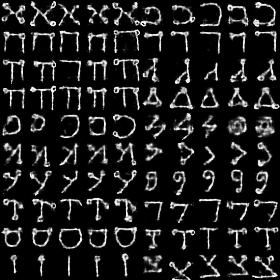}
\end{minipage}%
}%
\subfigure[Train Set]{
\begin{minipage}[t]{0.3\linewidth}
\centering
\includegraphics[width=1.0\linewidth]{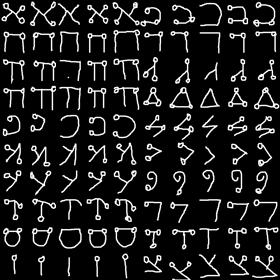}
\end{minipage}%
}%
\subfigure[$\alpha=5.0$]{
\begin{minipage}[t]{0.3\linewidth}
\centering
\includegraphics[width=1.0\linewidth]{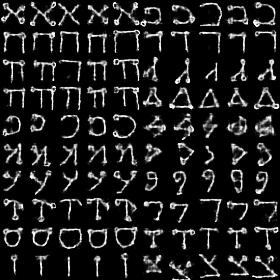}
\end{minipage}
}%
\centering
\caption{Compare the effects of image reconstruction for Omniglot dataset when $\alpha=0.2$ and $\alpha=5.0$. 
(a)$\alpha=0.2$, reconstruction image in Test sets (b)original image in Test sets
(c)$\alpha=5.0$, reconstruction image in Test sets (d) $\alpha=0.2$, reconstruction image in train sets 
(e) original image in train sets (f) $\alpha=5.0$, reconstruction image in train sets}
\label{fig-7-1}
\end{figure}

\end{document}